\documentclass[twocolumn]{article}

\usepackage{times}
\usepackage{float}
\usepackage{hyperref}
\usepackage{amsthm}
\usepackage{amsmath}
\usepackage{amssymb}
\usepackage{graphicx}
\usepackage{fullpage}
\usepackage{balance}
\usepackage{algorithm}
\usepackage[noend]{algpseudocode}
\usepackage{multirow}

\theoremstyle{plain}
\newtheorem{thm}{\protect\theoremname}
\theoremstyle{plain}
\newtheorem{lem}{\protect\lemmaname}
\theoremstyle{plain}
\newtheorem{cor}{\protect\corollaryname}
\theoremstyle{definition}
\newtheorem{defn}{\protect\definitionname}

\providecommand{\corollaryname}{Corollary}
\providecommand{\lemmaname}{Lemma}
\providecommand{\theoremname}{Theorem}
\providecommand{\definitionname}{Definition}

\title{Constraint Satisfaction over Generalized Staircase Constraints}

\author{
\begin{tabular}{ccc}
	Shubhadip Mitra & Partha Dutta & Arnab Bhattacharya \\
	\url{smitr@iitk.ac.in} & \url{parthdut@in.ibm.com} & \url{arnabb@iitk.ac.in} \\
	{Dept. of Computer Science} & \multirow{2}{*}{IBM Research Lab,} & {Dept. of Computer Science} \\
	{and Engineering,} & {} & {and Engineering,} \\
	{Indian Institute of Technology,} & \multirow{2}{*}{Bangalore,} & {Indian Institute of Technology,} \\
	{Kanpur,} & {} & {Kanpur,} \\
	{India} & {India} & {India} \\
\end{tabular}
}

\date{}

\begin{document}

\maketitle

\begin{abstract}
	One of the key research interests in the area of Constraint Satisfaction
	Problem (CSP) is to identify tractable classes of constraints and develop
	efficient solutions for them. In this paper, we introduce \emph{generalized
	staircase (GS)} constraints which is an important generalization of one such
	tractable class found in the literature, namely, staircase constraints.  GS
	constraints are of two kinds, \emph{down staircase (DS)} and \emph{up
	staircase (US)}. We first examine several properties of GS constraints, and
	then show that arc consistency is sufficient to determine a solution to a
	CSP over DS constraints. Further, we propose an optimal $O(cd)$ time and
	space algorithm to compute arc consistency for GS constraints where $c$ is
	the number of constraints and $d$ is the size of the largest domain.  Next,
	observing that arc consistency is not necessary for solving a DSCSP, we
	propose a more efficient algorithm for solving it. With regard to US
	constraints, arc consistency is not known to be sufficient to determine a
	solution, and therefore, methods such as path consistency or variable
	elimination are required. Since arc consistency acts as a subroutine for
	these existing methods, replacing it by our optimal $O(cd)$ arc consistency
	algorithm produces a more efficient method for solving a USCSP. \\
\end{abstract}


\section{Introduction \label{sec:Intro}}

Constraint Satisfaction Problem (CSP) is one of the key techniques in artificial
intelligence that offers a simple formal framework to represent and solve
several problems in temporal reasoning, job scheduling, pattern matching and
natural language processing  \cite{Tsang93foundationsof}.  Interestingly, in
many such problems, the underlying constraints are of a specific type. For this
reason, one of the key research interests in the area of CSP is to identify
tractable classes of constraints and develop efficient solutions for them.  In
this paper, we introduce \emph{generalized staircase (GS)} constraints which is
an important generalization of one such tractable class introduced in
\cite{CSPoverConnectedRowConvex}, namely, staircase constraints. 

\begin{figure}[t]
\begin{center}
\includegraphics[width=\columnwidth]{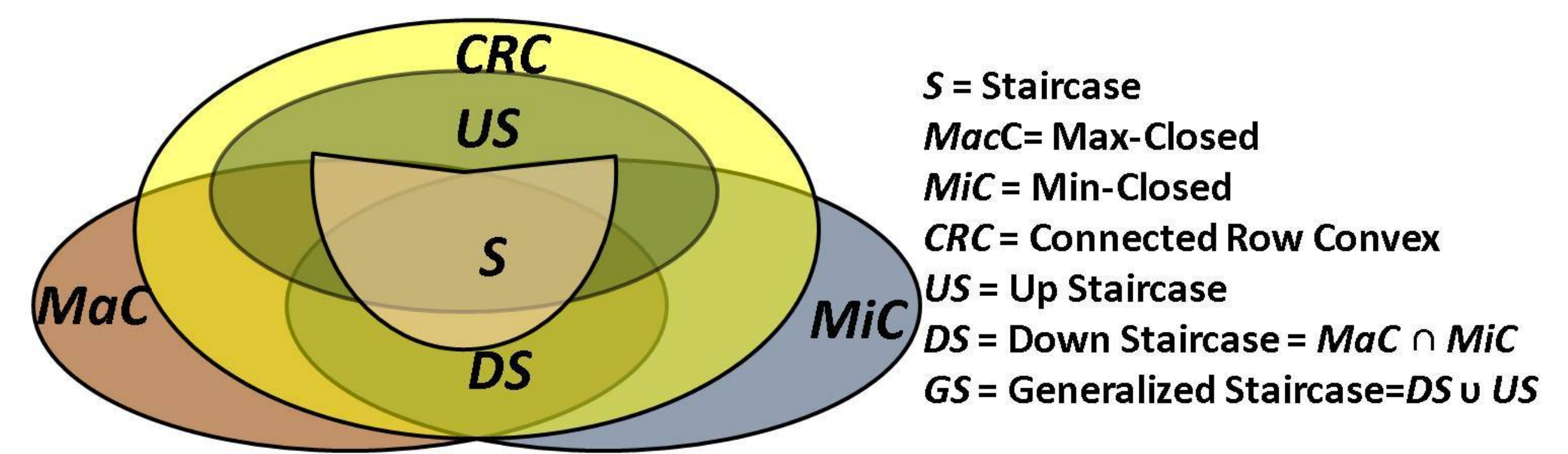}
\end{center}
\vspace*{-3mm}
\caption{The space of generalized staircase constraints.}
\vspace*{-3mm}
\label{fig:relationship}
\end{figure}

Figure~\ref{fig:relationship} shows the relationships of GS constraints with
other constraint classes. The class of GS constraints is the union of two
constraint classes, namely, \emph{down staircase (DS)} and \emph{up staircase
(US)}. While GS constraints \emph{strictly generalizes} the class of staircase
constraints ($S$, shown as a pie-shape), they are a \emph{strict subclass} of
connected row convex (CRC) constraints \cite{CSPoverConnectedRowConvex}.
Although related constraint classes such as CRC, max-closed, min-closed and
monotone constraints have been well studied in the literature, to the best of
our knowledge, the specific class of GS constraints have not been investigated
earlier. GS constraints are interesting to study as many temporal constraints
involving bounded intervals that arise in temporal reasoning and temporal
databases
\cite{Marios2007:continuous,Shubhadip2009,RelativeTemporalRETE,Wang:EDBT2006}
are often GS constraints. Due to their monotonic and convex structure, they
admit significantly simpler and faster solutions than CRC constraints.

A GS constraint can be either a down staircase (DS) or an up staircase (US)
constraint.  A DS (respectively, US) constraint is a CRC constraint
\cite{CSPoverConnectedRowConvex} such that if $[p,q]$ and $[p^\prime,q^\prime]$
are images of consecutive rows in its constraint matrix, then $p \le p^\prime
\wedge q \le q^\prime$ (respectively, $p \ge p^\prime \wedge q \ge q^\prime$).
Referring to the constraint matrices shown in Figure~\ref{fig:staircase}, a DS
(respectively, US) constraint has images of consecutive rows shifted towards
right (respectively, left) as we move down the matrix.

\begin{figure}[t]
\centering
\includegraphics[width=0.90\columnwidth]{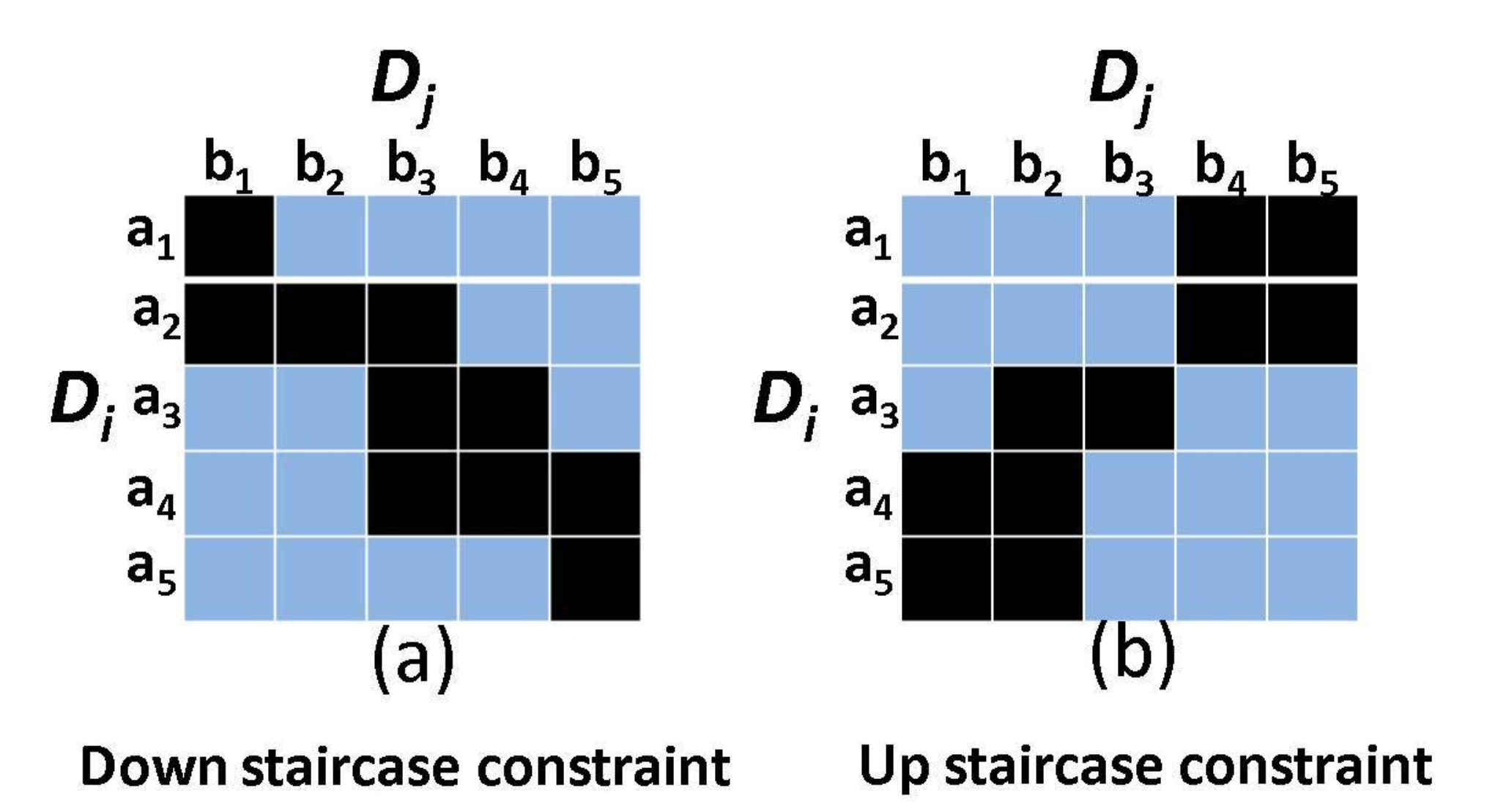}
\vspace*{-3mm}
\caption{Types of generalized staircase (GS) constraints.}
\vspace*{-3mm}
\label{fig:staircase}
\end{figure}

\subsection{Motivation}

Consider the following simplified problem of identifying cyclones at a given
location from weather records. Suppose a cyclonic instance is characterized by
the following set of events: hourly rainfall exceeding $\alpha$ units (event
$A$), hourly mean wind speed exceeding $\beta$ units (event $B$), and clockwise
wind direction (event $C$), such that the following constraints hold: (i)~$-3
\le A - B \le 1$, (ii)~$-2 \le B - C \le 2$, (iii)~$-2 \le C - A \le
3$ where $X$ denotes the time point of event $X \in \left \{A,B,C \right \}$.
These constraints essentially put a bound on the proximity of two different
events irrespective of their order of occurrence. For example, constraint~(i)
indicates that any $A$ event should start in the neighborhood of $B$ given by
$[B.s-3,B.s+1]$.

The weather database contains
a list of instances of events $A$, $B$ and $C$ sorted on $x$ where $x$ denotes an
instance of event $X$.  The goal is to determine tuples of the form
$\left\langle a,b,c \right\rangle$  where $a,b,c$ are instances of events $A$, $B$
and $C$ respectively, such that \emph{all} the above constraints are satisfied.
Modeling it as a CSP, we consider variables $X_A,X_B,X_C$ for events $A,B,C$
respectively, each of which have domains as the time points of respective event
instances, and the constraints as shown in the constraint network in
Figure~\ref{fig:cyclone}.  Considering the constraint matrix for the first
constraint, if
$a_i,a_{i^\prime} \in A$ have images $\left[b_j,b_k\right]$ and
$\left[b_{j^\prime},b_{k^\prime}\right]$ respectively in $B$ (i.e., $b_j$ to
$b_k$ satisfy the constraint for $a_i$, etc.), and $i \le i^\prime$, then $j \le
j^\prime$ and $k \le k^\prime$ (the subscripts denote the indices of the
instances in the sorted order).  This is the down staircase property as shown in
Figure~\ref{fig:staircase}(a). Similarly, the other constraints (ii) and (iii)
are also DS constraints.

Problems such as the above are common in complex event processing applications
\cite{SIGMOD2008PatternMatching,Bhargavi2010:intrusion,DistributedESPusingNFA,optimizingCEPRFID,ProbabilisticCEP,2010tesla,Shubhadip2009,CEPmulti-granularityRFID,Pietzuch2004:network,2009distributedCEPQR,Teymourian2009:stock,Vairo2010:WSN,RelativeTemporalRETE,Wang:EDBT2006,Wang:2010:SPNetwork,Wang2009:RFID,Xingyi2008:RFID,languageCEPRFID}
where the goal is to efficiently detect occurrences of complex events which are
usually represented as patterns of events sharing some temporal relationship
with each other.  Examples of complex events include state changes in business
and industrial processes, problems in enterprise systems and state changes in
the environment.  However, in most of the existing literature, due to the
simplicity of the patterns, they are represented using constructs similar to
regular expressions and the detection of complex events is realized by using
finite state automata
\cite{SIGMOD2008PatternMatching,DistributedESPusingNFA,optimizingCEPRFID,ProbabilisticCEP,2010tesla,CEPmulti-granularityRFID,Pietzuch2004:network,2009distributedCEPQR,RelativeTemporalRETE,languageCEPRFID}.
We observe that while regular expressions can capture sequential (or close to
sequential) patterns quite well, they are inadequate to capture arbitrary
pattern of temporal events sharing temporal relationships between pairs of
events. Hence, we consider a CSP based approach where the event pattern is
represented through a constraint network. In addition, we observe that the
constraints involved in such constraint networks are usually DS constraints, as
in the above example.

Another application of DSCSP is scheduling (or time-tabling) of resources (or
facilities) based on their availability. There is a variable for each resource
whose domain consists of the times of availability of the given resource. The
constraints involved are usually unary or binary. The unary constraints are
typically simple to process, and thus, domain values that do not satisfy those
are easy to eliminate. The binary constraints involve bounded temporal
intervals, which can be expressed as DS constraints. Solving this constraint
network would offer a feasible schedule for the resources. Such constraint-based
timetabling have been considered in
\cite{TimeTablePlanning,UniversityTimeTabling,ConstraintBasedTimetabling}.

\begin{figure}[t]
\centering
\includegraphics[width=\columnwidth]{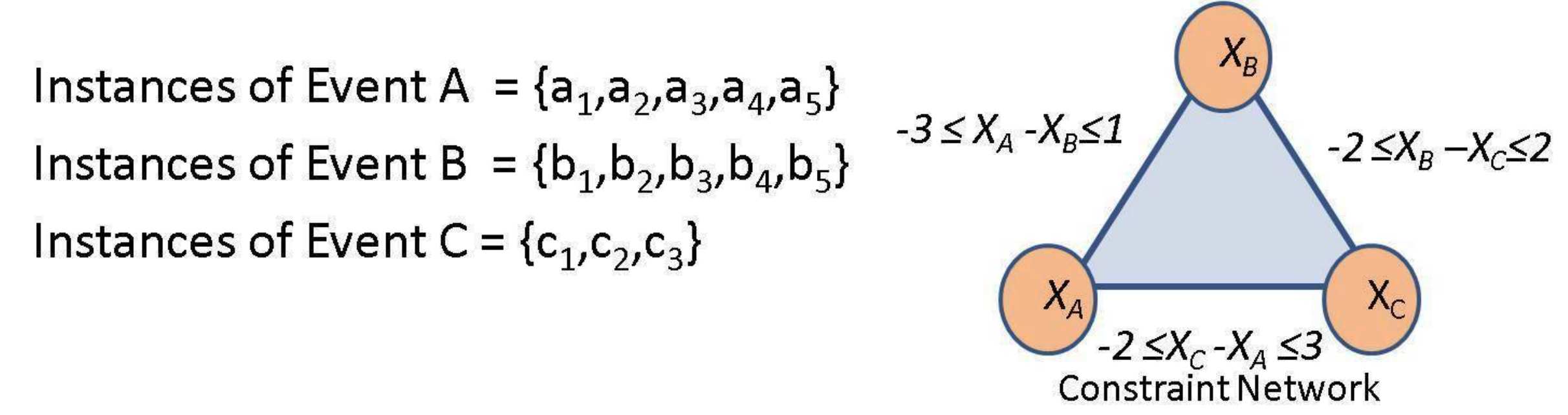}
\vspace*{-3mm}
\caption{Constraints for the example problem of cyclone identification.}
\vspace*{-3mm}
\label{fig:cyclone}
\end{figure}

\subsection{Contributions}

In this paper, we make the following contributions:
\begin{enumerate}
	\item We examine the structure and properties of GS constraints, besides
		studying its relationships with the existing constraint classes.
	\item We show that, for DS constraints, arc consistency is \emph{sufficient}
		to determine a solution.  Further, we present an optimal algorithm,
		called \emph{ACiDS}, for computing arc consistency for GS constraints in
		$O(cd)$ time and space, where there are $n$ variables, $c$ constraints
		and the largest domain size is $d$.
	\item Observing that arc consistency is \emph{not necessary} for solving a
		DSCSP, we propose a more efficient algorithm, \emph{DSCSP Solver}, for
		the same.
	\item With regard to US constraints, arc consistency is \emph{not
		sufficient} to determine a solution, and therefore, methods such as path
		consistency or variable elimination are required. Since arc consistency
		acts as a subroutine for these existing methods, replacing any known arc
		consistency algorithm by our optimal $O(cd)$ algorithm produces a more
		efficient method for solving a USCSP.
\end{enumerate}
	
\subsection{Organization}

The rest of the paper is organized as follows. Section~\ref{sec:Preliminaries}
covers the preliminary concepts of constraint satisfaction and formally defines
the class of GS constraints. Section~\ref{sec:work} outlines the related work in
this area. We next discuss the relationships of GS constraints with the existing
constraint classes and its closure properties in Section~\ref{sec:properties}.
In Section~\ref{sec:Arc-Consistency-for DSCSP}, we first show that arc
consistency is sufficient to solve a DSCSP and later present an efficient arc
consistency based algorithm, ACiDS, to solve a DSCSP. We also discuss how the
ACiDS algorithm is useful to solve a USCSP, although arc consistency is not
sufficient to solve it. In Section~\ref{sec:A faster algorithm}, we present a
more efficient algorithm, DSCSP Solver, to solve a DSCSP. Finally, we conclude
in Section~\ref{sec:conclusion}.

\section{Preliminaries\label{sec:Preliminaries}}

\subsection{CSP over GS Constraints}

A finite binary constraint satisfaction problem (CSP)
$\mathcal{P=\left(X,D,C\right)}$ is defined as a set of $n$ \emph{variables}
$\mathcal{X}=\left\{X_{1},X_{2},\dots,X_{n}\right\}$, a set of \emph{domains}
$\mathcal{D}=\left\{D_{1},D_{2},\dots,D_{n}\right\}$ where $D_{i}$ is the finite
set of possible values that can be assigned to $X_{i}$, and a set of $c$
\emph{constraints} $\mathcal{C}=\left\{C_{ij}\right\}$ where $C_{ij}$ is a
\emph{binary} constraint involving $X_{i}$ and $X_{j}$.  We assume that the
domain of each variable consists of at most $d$ distinct integers sorted in an
increasing order.  A constraint $C_{ij}$ on the ordered set of variables
$(X_{i},X_{j})$ specifies the admissible combinations of values for $X_{i}$ and
$X_{j}$. If a pair of values $(v_i,v_j)$ $(v_i\in D_i,v_j\in D_j)$ satisfy the
constraint $C_{ij}$, then we denote it as $(v_i,v_j)\in C_{ij}$.  For the sake
of simplicity, we assume that there exists \emph{at most} one constraint between
any pair of variables.   Given a constraint $C_{ij}$, the constraint $C_{ji}$
refers to a transposition of $C_{ij}$, i.e., if $(v_i,v_j)\in C_{ij}$, then
$(v_j,v_i)\in C_{ji}$.  Further we assume that all constraints are
\emph{binary}, i.e., they involve only two variables.  It was stated in
\cite{Survey92} that any $r$-ary CSP for any fixed $r \ge 1$ can be converted to
an equivalent binary CSP \cite{Rossi90equiv_CSP}. Therefore, it is reasonable to
consider CSPs with only binary constraints.  A binary constraint can always be
represented by a constraint matrix \cite{Montanari197495}, showing the
admissible combinations of values.

A CSP $\mathcal{P}$ asks whether there exists an assignment of the variables
$\mathcal{X}$ such that \emph{all} the constraints are satisfied.  If an
assignment $X_{i}=v_{i},\ \forall i$ is a solution, then $v_{i}$ is said to be a
\emph{member} of a solution.

Consider the set of vertices $\mathcal{V_P} = \left\{ V_1,V_2,\dots,V_n\right\}$
where $V_i$ corresponds to variable $X_i$. A pair of variables $(X_{i},X_{j})$
that has a constraint $C_{ij}$ is referred to by the \emph{arc} $(i,j)$ which
represents an edge between $V_i$ and $V_j$. The set of all arcs is denoted by
$arc(\mathcal{P})$. The graph induced by the vertex set $V_P$ and the edge set
$arc(\mathcal{P})$ is referred to as the \emph{constraint network} of
$\mathcal{P}$ and is denoted by $\mathcal{N_P}$.

Following \cite{CSPoverConnectedRowConvex}, a constraint $C_{ij}$ is
\emph{row-convex} if and only if in each row of the matrix representation of
$C_{ij}$, all the ones are consecutive, i.e., no two ones within a single row
are separated by a zero in that same row.  The \emph{reduced form} of a
constraint $C_{ij}$, denoted by $C_{ij}^{*}$, is obtained by removing all the
empty rows and columns in its matrix representation. The domain of $X_i$ through
the constraint $C_{ij}$, denoted by $D_i(C_{ij})$, is the set $\left\{v_i\in
D_i|\exists v_j:\left(v_i,v_j\right) \in C_{ij}\right\}$.  Suppose $C_{ij}$ is a
row-convex constraint and $v_i\in D_i(C_{ij})$. The image of $v_i$ in $C_{ij}$
is the set $\left\{v_j\in D_j|\left(v_i,v_j\right) \in C_{ij}\right\}$. Due to
row convexity of $C_{ij}$, this set can be represented as an interval
$[w_1,w_m]$ (over the domain $D_j(C_{ji})$) and we denote $w_1$ and $w_m$ by
$min(C_{ij},v_j)$ and $max(C_{ij},v_j)$ respectively.  We denote by
$succ(v_j,D_j(C_{ji}))$ and $pred(v_j,D_j(C_{ji}))$ the successor and the
predecessor of $v_j$ in $D_j(C_{ji})$ respectively. For ease of notation, we
simply use $min(v_j),max(v_j),succ(v_j)$ and $pred(v_j)$ when there is no
ambiguity on the underlying domain.

\begin{figure}[t]
\begin{center}
\includegraphics[width=0.80\columnwidth]{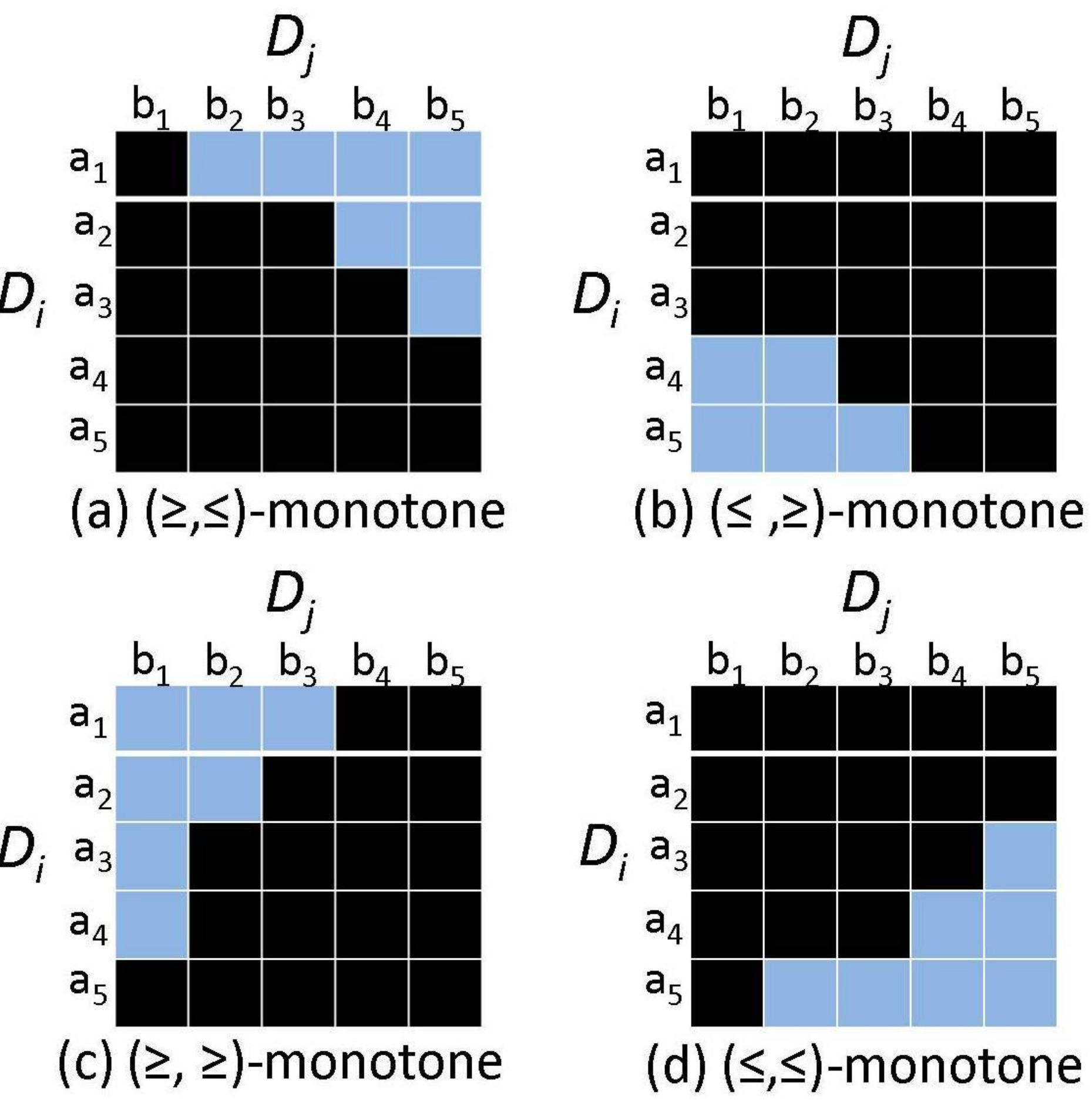}
\end{center}
\vspace*{-3mm}
\caption{$(\alpha,\beta)$-monotone constraints.}
\vspace*{-3mm}
\label{fig:alphabetamonotone}
\end{figure}

Referring to the example in Figure~\ref{fig:staircase}(a),
$min(a_3)=b_3,max(a_3)=b_4,succ(b_3)=b_4$ and $pred(b_3)=b_2$.  A row-convex
constraint $C_{ij}$ is connected row-convex (CRC) if and only if the images
$[a,b]$ and $[a^{\prime},b^{\prime}]$ of two consecutive rows in $C_{ij}^{*}$
are such that $a^\prime \leq succ(b) \wedge b^\prime \geq pred(a)$. Following
\cite{Beek95RowConvex}, a constraint $C_{ij}$ is said to be \emph{monotone} if
and only if $(v,w)\in C_{ij}$ implies $\left(v_i^\prime,v_j^\prime \right)\in
C_{ij}$ for all values $v_i,v_i^\prime \in D_i$, $v_j,v_j^\prime \in D_j$ such
that $v_i^\prime \geq v_i$ and $v_j^\prime \leq v_j$, i.e., referring to the
constraint matrix of $C_{ij}$, if $(v_i,v_j)\in C_{ij}$, then all cells to the
left of $(v_i,v_j)$ on the same row are in $C_{ij}$ and all cells below
$(v_i,v_j)$ on the same column are in $C_{ij}$.  However, a complementary
definition of monotone constraints was proposed in \cite{GenericAC5}, where a
constraint $C_{ij}$ is said to be \emph{monotone} if and only if $(v_i,v_j)\in
C_{ij}$ implies $\left(v_i^\prime,v_j^\prime \right)\in C_{ij}$ for all values
$v_i,v_i^\prime \in D_i$, $v_j,v_j^\prime \in D_j$ such that $v_i^\prime \leq
v_i$ and $v_j^\prime \geq v_j$, i.e., referring to the constraint matrix of
$C_{ij}$, if $(v_i,v_j)\in C_{ij}$, then all cells to the right of $(v_i,v_j)$
on the same row are in $C_{ij}$ and all cells above $(v_i,v_j)$ on the same
column are in $C_{ij}$. An example of each of them is shown in
Figure~\ref{fig:alphabetamonotone}(a) and Figure~\ref{fig:alphabetamonotone}(b).
Now we consider a generalization of monotone constraints, introduced in
\cite{CSPoverConnectedRowConvex}, as following.

\begin{defn}
	Let $\preceq$ and $\succeq$ be total orderings on $D_i$ and $D_j$,
	respectively. A constraint $C_{ij}$ is $(\preceq,\succeq)$-\emph{monotone}
	if and only if $(v_i,v_j)\in C_{ij}$ implies $\left(v_i^\prime,v_j^\prime
	\right)\in C_{ij}$ for all values $v_i,v_i^\prime \in D_i$, $v_j,v_j^\prime
	\in D_j$ such that $v_i^\prime \preceq v_i$ and $v_j^\prime \succeq v_j$.
	\emph{Staircase} constraints are $(\alpha,\beta)$-monotone constraints where
	$\alpha,\beta \in \{\leq,\geq\}$.
\end{defn}

Figure~\ref{fig:alphabetamonotone} shows examples of staircase constraints.

In this paper, we propose \emph{generalized staircase (GS)} constraints as a
generalization of  staircase constraints.  GS constraints are of two types, down
staircase (DS) and up staircase (US), defined as below.

\begin{defn}
	A constraint $C_{ij}$ is a \emph{down staircase (DS)} constraint if and only
	if it is row convex and for any $u,v \in D_i$ such that
	$v=succ\left(u\right)$, the following conditions (DS property) hold:
	$min\left(C_{ij},u\right) \leq min\left(C_{ij},v\right)$ and
	$max\left(C_{ij},u\right) \leq max\left(C_{ij},v\right)$. 
\end{defn}

\begin{defn}
	A constraint $C_{ij}$ is an \emph{up staircase (US)} constraint if and only
	if it is row convex and for any $u,v \in D_i$ such that
	$v=succ\left(u\right)$, the following conditions (US property) hold:
	$min\left(C_{ij},u\right) \geq min\left(C_{ij},v\right)$ and
	$max\left(C_{ij},u\right) \geq max\left(C_{ij},v\right)$.
\end{defn}

Figure~\ref{fig:staircase} shows examples of both the constraints.

Following the above definitions, it is clear that this class of GS constraints
is a strict subclass of CRC constraints. Further, $(\leq,\geq)$-monotone and
$(\geq,\leq)$-monotone constraints are strict subclasses of DS constraints
while $(\leq,\leq)$-monotone and $(\geq,\geq)$-monotone are strict subclasses
of US constraints (see Figure~\ref{fig:relationship}).  A binary CSP with all
its constraints as DS (respectively US) constraints is referred to as a
\emph{DSCSP} (respectively \emph{USCSP}).

\subsection{Arc Consistency and Path Consistency}

Following \cite{GenericAC5}, an $arc(i,j)$ is \emph{arc consistent} if and only
if $\forall v_i\in D_i, \exists v_j \in D_j$ such that $(v_i,v_j)\in C_{ij}$
($v_j$ is said to \emph{support} $v_i$). For arc consistency, we assume that
constraint $C_{ij} \in \mathcal{C}$ if and only if $C_{ji} \in \mathcal{C}$,
i.e., $arc(i,j)$ and $arc(j,i)$ are treated distinct, and that both are in
$arc(\mathcal{P})$.  If each $arc(i,j)\in arc(\mathcal{P})$ is arc consistent,
then the constraint network  $\mathcal{N_P}$ is said to be \emph{arc
consistent}.

Following \cite{CSPoverConnectedRowConvex}, a constraint network $\mathcal{N_P}$
is  \emph{path consistent} if and only if for every triple $(X_i,X_k,X_j)$ of
variables such that $arc(i,k),arc(k,j),arc(i,j)\in arc(\mathcal{P})$, for every
$v_i \in D_i$ and $v_j \in D_j$ such that $(v_i,v_j) \in C_{ij}$, there exists
$v_k \in D_k$ such that $(v_i,v_k) \in C_{ik}$ and $(v_k, v_j) \in C_{kj}$.

Two constraint networks are \emph{equivalent} if they have the same set of
solutions. Given a constraint network, the task of an arc consistency
(respectively path consistency) algorithm is to generate an equivalent
constraint network that is  arc consistent (respectively path consistent).

\section{Related Work}
\label{sec:work}

Although CSPs are NP-hard, several tractable classes of CSPs have been
identified \cite{Pearson97asurvey}.  Among the notable tractable classes of
constraints are functional, anti-functional, monotone, min-closed, max-closed and
connected row convex (CRC) constraints. In this paper, we are considering
another important tractable class of generalized staircase (GS) constraints.

In this work, we propose an algorithm named ACiDS, for computing arc consistency
over DS constraints. In this context, we note that there are several existing
arc consistency algorithms for general arbitrary constraints. The most prominent
ones include AC3 \cite{Mackworth1977:AC3}, AC4 \cite{Mohr_Henderson1986:AC4},
AC5 \cite{GenericAC5}, AC6 \cite{AC6} and AC7 \cite{AC7}. Although AC3 has
non-optimal worst case time complexity ($O(cd^3)$), it is still considered as
the basic constraint propagation algorithm owing to its simplicity. Its
successors AC4, AC5, AC6 and AC7 have optimal worst case time complexity of
$O(cd^2)$. AC5 is a generic arc consistency algorithm that offers a framework to
compute arc consistency for any constraint network and can be optimized for
certain classes of constraints.  A refinement of AC3 in the form of AC2001 was
proposed in \cite{AC2001} which has an optimal worst case time complexity of
$O(cd^2)$. None of these algorithms assume knowledge of the semantics of the
underlying constraints. Although they agree on the worst case time complexity
(besides AC3), their performance varies for different constraint classes. In
addition, the tightness and slackness of the constraints also affect the
performance for any given constraint class.  Since arc consistency is sufficient
to solve a DSCSP, as shown in Theorem~\ref{thm:ACiDS}, any DSCSP can be solved
by any of these arc consistency algorithms (AC4, AC6, AC7 or AC2001) with a
worst case time complexity of $O(cd^2)$. The arc consistency algorithm presented
in this paper, namely ACiDS, has an optimal worst case time complexity of
$O(cd)$ for the class of DS constraints, thus improving the $O(cd^2)$ bound for
solving DS constraints using the existing arc consistency algorithms. This
algorithm is a specialization of the AC5 algorithm. Similar specializations for
several classes of arithmetic constraints such as functional, anti-functional
and monotone constraints were stated in \cite{GenericAC5}. For each of these
classes, AC5 determines a solution in $O(cd)$ optimal time.

GS constraints are a subclass of CRC constraints. While path consistency is
sufficient to solve a CSP over CRC constraints, we show that arc consistency is
sufficient to solve a CSP over DS constraints (please refer to
Theorem~\ref{thm:ACiDS}). However, it is not known that
whether US constraints are solvable using arc consistency. For the class of CRC
constraints,  a path consistency based algorithm was proposed in
\cite{CSPoverConnectedRowConvex} that runs in $O(n^3d^2)$ time. Based on
variable-elimination, \cite{AAAI:2007RowConvex} stated an
$O\left(nd\sigma^{2}+cd^{2}\right)$ time algorithm for CRC constraints where
$\sigma\le n$ is the degree of elimination of the triangulated graph of the
underlying constraint network. Both these algorithms for CRC constraints use arc
consistency as a subroutine which takes $O(cd^2)$ time.  US constraints being a
subclass of CRC constraints can be solved by either of these techniques
mentioned above. Our arc consistency algorithm for GS constraints that runs in
$O(cd)$ time, acting as a subroutine, thus improves the time complexity of the
algorithm by \cite{AAAI:2007RowConvex} to $O\left(nd\sigma^{2}+cd\right)$ (which
is linear in $d$), when applied to US constraints.

Referring to the classes of max-closed and min-closed constraints,
\cite{TractableConstraints95} showed that their solutions can be computed in
$O\left(c^{2}d^{2}\triangle^{2}\right)$ time where $\triangle \le d$ is the
maximum number of supports that a value in a domain can have. In this paper, we
show that a constraint is a DS if and only if it is both max-closed and
min-closed. Our DS algorithms have significantly lower time complexities for
this restricted subclass.

\section{Properties of GS Constraints}
\label{sec:properties}

\subsection{Relationships with other Constraint Classes \label{sec:relationship}}

In this section, we discuss the relationship of GS constraints with the existing
constraint classes. Figure~\ref{fig:relationship} illustrates the space of GS
constraints. Since some of the relationships of GS constraints shown in this
figure have already been covered so far, we next explain those that have not yet
been discussed.

Following \cite{TractableConstraints95}, a constraint $C$ is \emph{max-closed}
if and only if $\forall (u,v), (u^\prime,v^\prime) \in C$, implies
$(\max\{u,u^\prime\},\max\{v,v^\prime\})\in C$. Similarly, a constraint $C$ is
\emph{min-closed} if and only if $\forall (u,v), (u^\prime,v^\prime) \in C$,
implies $(\min\{u,u^\prime\},\min\{v,v^\prime\})\in C$. The authors in
\cite{TractableConstraints95} claimed that if a constraint is both max-closed
and min-closed, then it is row convex. If a constraint is either a
$(\le,\ge)$-monotone or $(\ge,\le)$-monotone, then it is both max-closed and
min-closed. We integrate these claims in the following result.

\begin{thm} \label{thm:MinMaxClosed}
	A constraint is down staircase if and only if it is both max-closed and
	min-closed.
\end{thm}

\begin{proof}
	Suppose $C_{ij}$ is a constraint that is both max-closed and min-closed, and
	therefore, row convex \cite{TractableConstraints95}. For any $u,u^\prime\in
	D_i$ such that $u^\prime = succ(u)$, the images of $u,u^\prime$ in $D_j$ are
	intervals, say $[v,w]$ and $[v^\prime,w^\prime]$.  Since $u < u^\prime$, we
	claim that $v \le v^\prime$ and $w \le w^\prime$.  If $v^\prime<v$, then
	using the min-closed property $(u,v^\prime)\in C_{ij}$. This contradicts the
	assumption that the image of $u$ in $D_j$ is $[v,w]$. If $w^\prime <w$, then
	using the max-closed property $(u^\prime,w) \in C_{ij}$, thus contradicting
	the assumption that $u^\prime$ has the image $[v^\prime,w^\prime]$ in $D_j$.
	Hence, $C_{ij}$ must be a down staircase constraint.

	Now, suppose $C_{ij}$ is a down staircase constraint.  From the down
	staircase property, for any $u,u^\prime\in D_i$ such that $u<u^\prime $, if
	the images of $u,u^\prime$ in $D_j$ are $[v,w]$ and $[v^\prime,w^\prime]$,
	then $v\le v^\prime$ and $w\le w^\prime$. If $w \le v^\prime$, both
	max-closed and min-closed properties are trivially satisfied.
	
	Next, consider the case $v\le v^\prime<w \le w^\prime$, and any $y,y^\prime
	\in D_j$ such that $v\le y \le w$ and $v^\prime \le y^\prime \le w^\prime$.
	If $y \le y^\prime$, then the min-closed and max-closed properties are
	trivially satisfied for any pair $(u,y),(u^\prime,y^\prime) \in C_{ij}$.
	Hence, assume $y>y^\prime$.  However, such $y^\prime$ exists only for
	$succ(v^\prime) \le y \le w$. Observe that $(u,y^\prime)\in C_{ij}$ for
	$v^\prime \le y^\prime<y$. Also $(u^\prime,y)\in C_{ij}$ for any
	$v^\prime\le y^\prime < y \le w$. Hence, both min-closed and max-closed
	properties are satisfied for any pair $(u,y),(u^\prime,y^\prime) \in
	C_{ij}$. 
\end{proof}

\begin{figure}[t]
\begin{center}
\includegraphics[width=0.90\columnwidth]{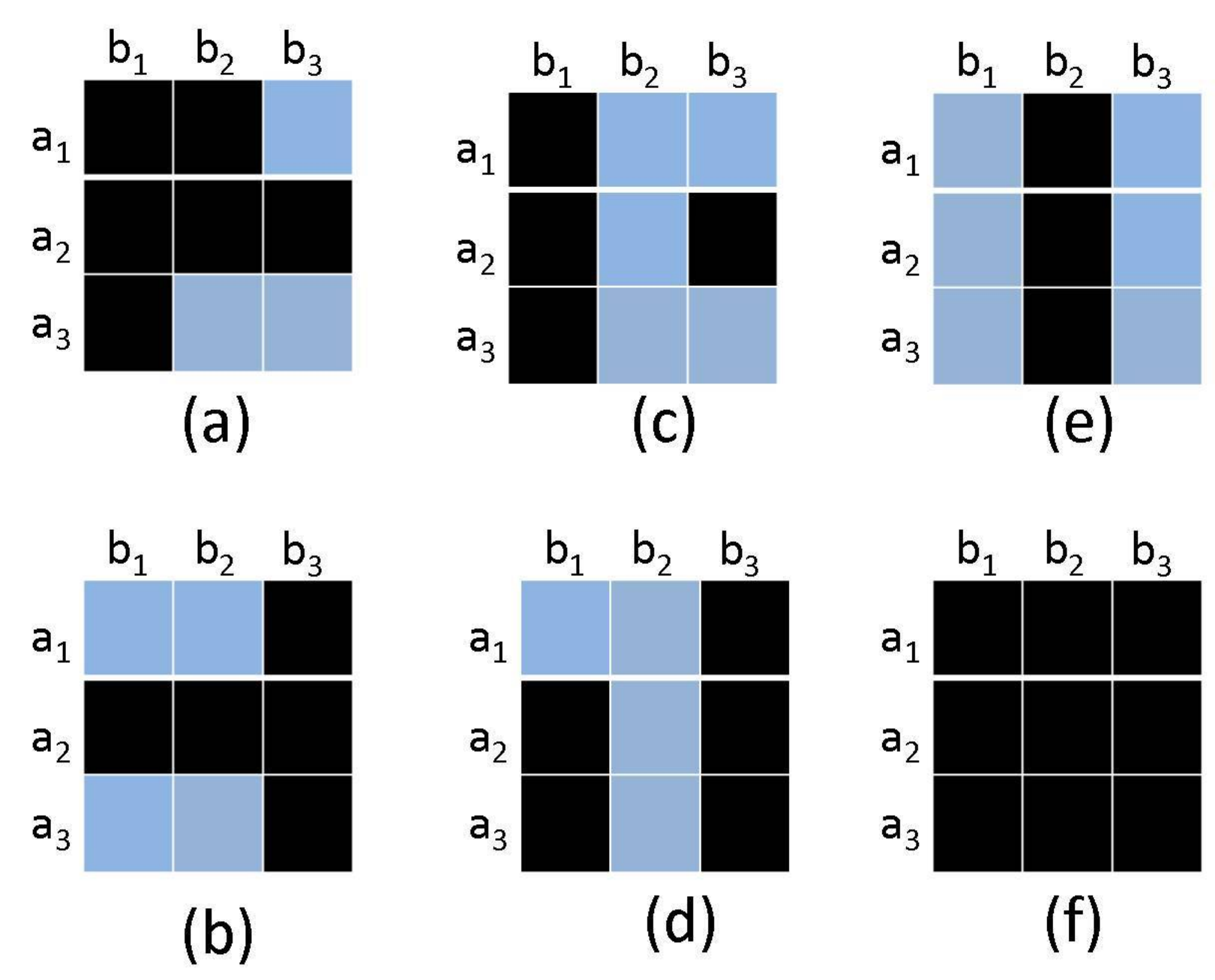}
\end{center}
\vspace*{-3mm}
\caption{Examples of max-closed and min-closed constraints.}
\vspace*{-3mm}
\label{fig:excons}
\end{figure}

Figure~\ref{fig:relationship} shows the intersection space of max-closed and
min-closed constraints as DS constraints. Now we consider some of the other
relationships shown in the figure. We note that there are CRC constraints that
are not GS, but min-closed (Figure~\ref{fig:excons}(a)), or max-closed
(Figure~\ref{fig:excons}(b)).  From Figure~\ref{fig:excons}(c) and
Figure~\ref{fig:excons}(d), we also see the existence of constraints that are
either min-closed or max-closed, but not CRC.  Figure~\ref{fig:excons}(e) and
Figure~\ref{fig:excons}(f) illustrate existence  of GS constraints that are both
DS and US. While Figure~\ref{fig:excons}(e) is not a staircase constraint,
Figure~\ref{fig:excons}(f) is staircase.  Figure~\ref{fig:alphabetamonotone}
shows that any staircase constraint is either max-closed, or min-closed, or
both. This implies that any US constraint that is also a staircase constraint,
is either max-closed or min-closed, but not both.  We also claim that if a US
constraint is either a max-closed or a min-closed constraint, it must be a
staircase constraint. However, we note that there are US constraints that are
neither max-closed nor min-closed.  Figure~\ref{fig:staircase}(b)  illustrates
this fact.

\subsection{Closure Properties of GS Constraints \label{sec:closure}}

Following \cite{CSPoverConnectedRowConvex}, we know that CRC constraints are
closed under the following operations: \emph{transposition}, \emph{intersection}
and \emph{composition}. Transposition of the constraint matrix implies that if
$(v_i,v_j) \in C_{ij}$, then $(v_j,v_i) \in C_{ji}$.  Intersection of two binary
constraints $C_{ij}$ and $C_{ij}^\prime$ implies that $(v_i,v_j) \in C_{ij} \cap
C_{ij}^\prime$ if and only if $(v_i,v_j)\in C_{ij} \wedge (v_i,v_j)\in
C_{ij}^\prime$. Composition refers to multiplication of two constraint matrices,
i.e., given two binary constraints $C_{ij}$ and $C_{jk}$, the constraint
$C_{ik}=C_{ij}\times C_{jk}$ is such that $(v_i,v_k)\in C_{ik}$ if and only if
$\exists v_j|(v_i,v_j)\in C_{ij}\wedge (v_j,v_k)\in C_{jk}$.

These closure properties are necessary for path consistency which ensures
tractability of CRC constraints.  In this section, we discuss the closure
properties for the subclasses of GS constraints, namely DS and US constraints.
We show that DS constraints are closed under transposition, intersection and
composition. US constraints are, however, closed under transposition and
intersection, but not under composition. In fact composition of two US
constraints result in a DS constraint. 

\begin{lem}
DS constraints are closed under transposition.
\end{lem}

\begin{proof}
Suppose $C_{ij}$ is a DS constraint, and $C_{ji}$ is its transpose.

From Theorem \ref{thm:MinMaxClosed}, we know that $C_{ij}$ is both min-closed
and max-closed. Hence, for any two tuples $(a,b),(c,d) \in C_{ij}$,
$(\min\{a,c\},\min\{b,d\})\in C_{ij}$ and $(\max\{a,c\},\max\{b,d\})\in C_{ij}$.
From the definition of transposition, for every $(u,v)\in C_{ij}$, $(v,u) \in
C_{ji}$. Therefore, for any two tuples $(b,a),(d,c) \in C_{ji}$,
$(\min\{b,d\},\min\{a,c\})\in C_{ji}$ and $(\max\{b,d\},\max\{a,c\})\in C_{ji}$.
Hence, $C_{ji}$ is both min-closed and max-closed. From Theorem
\ref{thm:MinMaxClosed}, we conclude that $C_{ji}$ is a DS constraint.
\end{proof}

\begin{lem}
DS constraints are closed under intersection.
\end{lem}

\begin{proof}
Suppose $C_{ij},C_{ij}^\prime$ are two DS constraints on the same domains $D_i$
and $D_j$. Assume that $A_{ij}=C_{ij}\cap C_{ij}^\prime$ denote the intersection
of $C_{ij}$ and $C_{ij}^\prime$. Suppose $A_{ij}$ is not a DS constraint.

Consider any two non-empty rows of $A_{ij}$, $u,v \in D_i$ such that $u<v$.
Suppose $u,v$ have the respective images $[a,b],[c,d]$ in $C_{ij}$. Similarly,
let $u,v$ have images $[a^\prime,b^\prime],[c^\prime,d^\prime]$ respectively in
$C_{ij}^\prime$. Note that $min(A_{ij},u)=\max\{a,a^\prime\}$ and
$max(A_{ij},u)=\min\{b,b^\prime\}$. Similarly,
$min(A_{ij},v)=\max\{c,c^\prime\}$ and $max(A_{ij},v)=\min\{d,d^\prime\}$.

From the closure properties of CRC constraints in
\cite{CSPoverConnectedRowConvex}, $A_{ij}$ must be a CRC constraint. For
$C_{ji}$ not to be a DS constraint, either of the following two cases must hold:
\\
(1)~$min(A_{ij},u)>min(A_{ij},v)$. From the DS property of $C_{ij}$, $a \le c$.
Similarly, $a^\prime \le c^\prime$. This, however, contradicts
$min(A_{ij},u)=\max\{a,a^\prime\}>\max\{c,c^\prime\}=min(A_{ij},v)$. \\
(2)~$max(A_{ij},u)>max(A_{ij},v)$. From the DS property of $C_{ij}$, $b \le d$
and $b^\prime \le d^\prime$. This, however, contradicts
$max(A_{ij},u)=\min\{b,b^\prime\}>\min\{d,d^\prime\}=max(A_{ij},v)$.
\end{proof}

\begin{lem}
US constraints are closed under transposition.
\end{lem}

\begin{proof}
Suppose $C_{ij}$ is a US constraint, and $C_{ji}$ is its transpose.
Suppose $C_{ji}$ is not a US constraint.

Consider any two rows of matrix representation of $C_{ji}$, say $u,v \in D_j$
such that $u<v$. Suppose the images of $u,v$ in $C_{ji}$ are $[a,b]$ and $[c,d]$
respectively where $a,b,c,d\in D_i$. From the closure properties of CRC
constraints in  \cite{CSPoverConnectedRowConvex}, $C_{ji}$ must be a CRC
constraint. For $C_{ji}$ not to be a US constraint, either of the following two
cases must hold: \\
(1)~$a<c$. From the definition of $max$ values, $v \le max(c)$. Since
$a<c=min(v)$ and $C_{ij}$ is row convex, $max(a)<v$. From the above relations,
it follows that $max(a)< max(c)$. This, however, contradicts the US property of
$C_{ij}$. \\
(2)~$b<d$. From the definition of $min$ values, $min(b) \le u$. Since
$b=max(u)<d$ and $C_{ij}$ is row convex, $u<min(d)$. From the above relations,
it follows that $min(b)< min(d)$. This, however, contradicts the US property of
$C_{ij}$. 
\end{proof}

\begin{lem}
US constraints are closed under intersection.
\end{lem}

\begin{proof}
Suppose $C_{ij},C_{ij}^\prime$ are two US constraints on the same domains $D_i$
and $D_j$. Assume that $A_{ij}=C_{ij}\cap C_{ij}^\prime$ denote the intersection
of $C_{ij}$ and $C_{ij}^\prime$. Suppose $A_{ij}$ is not a US constraint.

Consider any two non-empty rows of $A_{ij}$, $u,v \in D_i$ such that $u<v$.
Suppose $u,v$ have images $[a,b],[c,d]$ respectively in $C_{ij}$ and
$[a^\prime,b^\prime],[c^\prime,d^\prime]$ respectively in $C_{ij}^\prime$. Note
that $min(A_{ij},u)=\max\{a,a^\prime\}$ and $max(A_{ij},u)=\min\{b,b^\prime\}$.
Similarly, $min(A_{ij},v)=\max\{c,c^\prime\}$ and
$max(A_{ij},v)=\min\{d,d^\prime\}$.

From the closure properties of CRC constraints in
\cite{CSPoverConnectedRowConvex}, $A_{ij}$ must be a CRC constraint. For
$C_{ji}$ not to be a US constraint, either of the following two cases must hold:
\\
(1)~$min(A_{ij},u)<min(A_{ij},v)$. From the US property of $C_{ij}$, $a \ge c$.
Similarly, $a^\prime \ge c^\prime$. This, however, contradicts
$min(A_{ij},u)=\max\{a,a^\prime\}<\max\{c,c^\prime\}=min(A_{ij},v)$. \\
(2)~$max(A_{ij},u)<max(A_{ij},v)$. From the US property of $C_{ij}$, $b \ge d$
and $b^\prime \ge d^\prime$. This, however, contradicts
$max(A_{ij},u)=\min\{b,b^\prime\}<\min\{d,d^\prime\}=max(A_{ij},v)$.
\end{proof}

Next we show that the composition of two DS constraints result in another DS
constraint, as illustrated by an example in Figure~\ref{fig:compositionDS}. 

\begin{figure}[t]
\begin{center}
\includegraphics[width=\columnwidth]{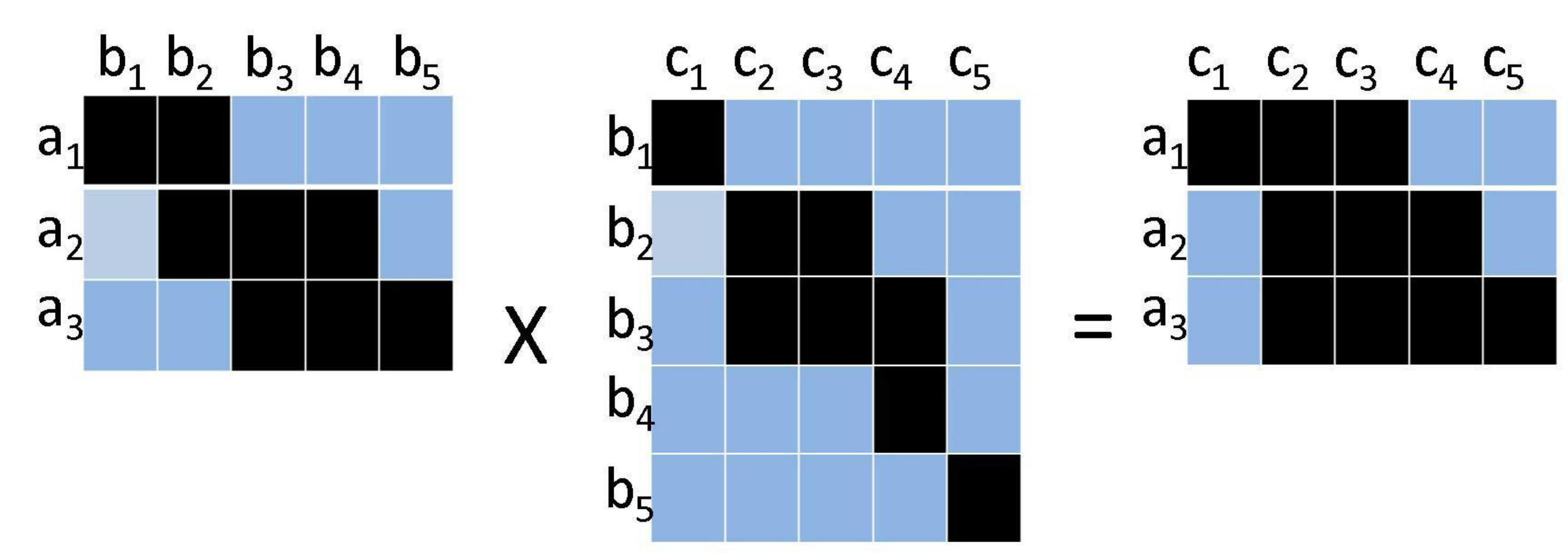}
\end{center}
\vspace*{-3mm}
\caption{Composition of two DS constraints.}
\vspace*{-3mm}
\label{fig:compositionDS}
\end{figure}

\begin{lem}
DS constraints are closed under composition.
\end{lem}

\begin{proof}
Suppose $C_{ik}=A_{ij} \times B_{jk}$ is the composition of two DS constraints
$A_{ij}$ and $B_{jk}$ with all non-empty rows and columns.

Suppose any $u \in D_i$ has image $[a,b]$ in $A_{ij}$. We claim that if
$c=min(B_{jk},a)$ and $d=max(B_{jk},b)$, then the image of $u$ in $C_{ik}$ is
$[c,d]$.

Consider any $v\in D_i$ such that $v>u$. By a similar argument, if $v$ has an
image $[a^\prime,b^\prime]$ in $A_{ij}$, then the image of $v$ in $C_{ik}$ is
$[c^\prime,d^\prime]$ where $c^\prime=min(B_{jk},a^\prime)$ and
$d^\prime=max(B_{jk},b^\prime)$.

Since $v>u$ and $A_{ij}$ is a DS constraint, $a^\prime \ge a$ and $b^\prime \ge
b$. Using these relations and the fact that $B_{jk}$ is a DS constraint, we
infer that $c\le c^\prime$ and $d \le d^\prime$.  Thus, for any $u,v \in D_i$
such that $u<v$, $min(C_{ik},u)\le min(C_{ik},v)$ and $max(C_{ik},u) \le
max(C_{ik},v)$. Hence, $C_{ik}$ is a DS constraint.
\end{proof}

Next, we show that composition of two US constraints result in a DS constraint,
as illustrated by an example in Figure~\ref{fig:compositionUS}.

\begin{lem}
Composition of US constraints results in a DS constraint.
\end{lem}

\begin{proof}
Suppose $C_{ik}=A_{ij} \times B_{jk}$ is the composition of two US constraints
$A_{ij}$ and $B_{jk}$ with all non-empty rows and columns.

Suppose any $u\in D_i$ has image $[a,b]$ in $A_{ij}$. We claim that if
$c=min(B_{jk},b)$ and $d=max(B_{jk},a)$, then the image of $u$ in $C_{ik}$ is
$[c,d]$.

Consider any $v\in D_i$ such that $v>u$. By a similar argument, if $v$ has an
image $[a^\prime,b^\prime]$ in $A_{ij}$, then the image of $v$ in $C_{ik}$ is
$[c^\prime,d^\prime]$ where $c^\prime=min(B_{jk},b^\prime)$ and
$d^\prime=max(B_{jk},a^\prime)$.

Since $v>u$, and $A_{ij}$ is an US constraint, $a^\prime \le a$ and $b^\prime
\le b$. Using these relations and the fact that $B_{jk}$ is an US constraint, we
observe that $c\le c^\prime$ and $d \le d^\prime$.  Thus, for any $u,v \in D_i$
such that $u<v$, $min(C_{ik},u)\le min(C_{ik},v)$ and $max(C_{ik},u) \le
max(C_{ik},v)$. Hence, $C_{ik}$ is a DS constraint.
\end{proof}

\begin{figure}[t]
\begin{center}
\includegraphics[width=\columnwidth]{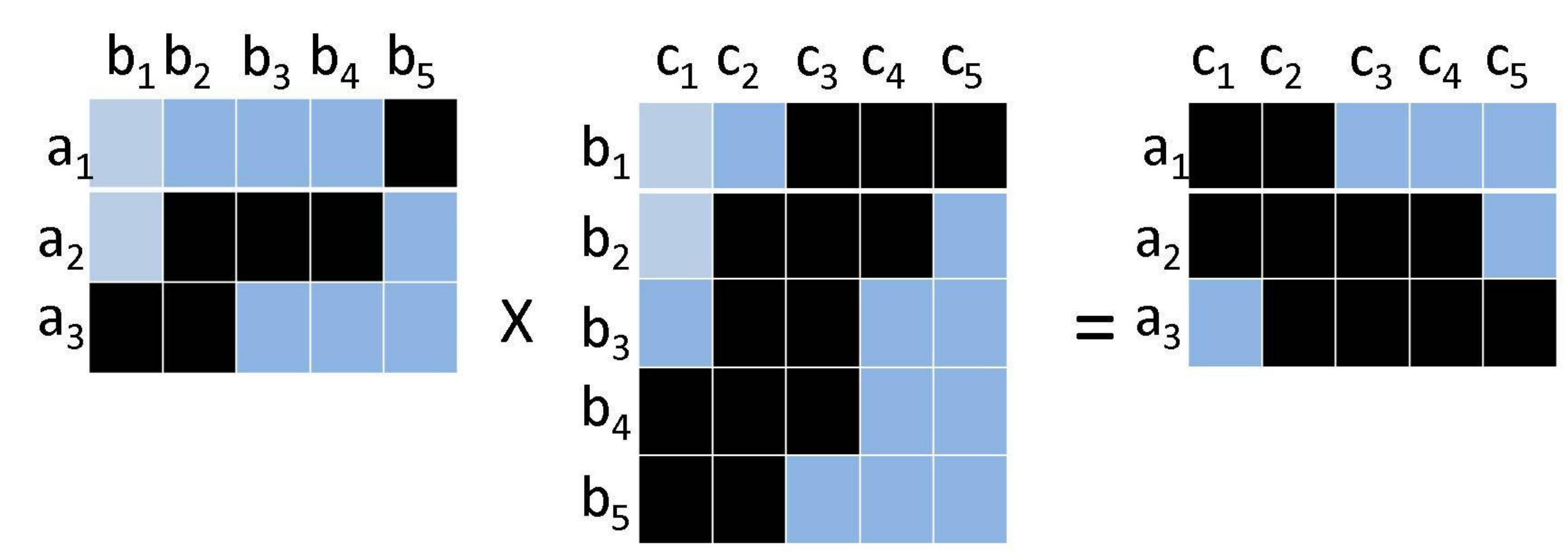}
\end{center}
\vspace*{-3mm}
\caption{Composition of two US constraints.}
\vspace*{-3mm}
\label{fig:compositionUS}
\end{figure}

Owing to monotonicity and row convexity, GS constraints admit much simpler and
efficient algorithms to compute the above operations, as opposed to those for
CRC constraints. From the arguments given in the above proofs, it is easy to see
that each of these operations can be performed in $O(d)$ space and time for both
DS and US constraints.  Since DS constraints are solvable using arc consistency
(as described in Section~\ref{sec:Arc-Consistency-for DSCSP}), we do not need
these closure properties. However, for US constraints, if we use either of the
existing approaches for CRC constraints (mentioned in Section~\ref{sec:work}),
we require these properties, and hence, lies the advantage of efficient
algorithms to compute these operations.

\section{Arc Consistency for GS Constraints}
\label{sec:Arc-Consistency-for DSCSP}

In this section, we first show that arc consistency is sufficient to determine a
solution to a DSCSP, following which we present an algorithm, \textsc{ACiDS}
(Arc Consistency for Down Staircase constraints), to compute arc consistency for
a DSCSP.  We also discuss how the same algorithm can be used for computing arc
consistency for a USCSP.

\subsection{Sufficiency of Arc Consistency for Solving DSCSP}

\begin{thm} \label{thm:ACiDS}
	Arc consistency is sufficient to determine a solution to a DSCSP.
\end{thm}

\begin{proof}
	Given a DSCSP $\mathcal{P}$, suppose $\forall i, D_i,D_i^\prime$
	respectively denote the domain of $X_i$ before and after achieving arc
	consistency, and $\forall i, f_i,l_i$ denote the first and last values of
	$D_i^\prime$ respectively $(f_i\leq l_i)$. We claim that the sets
	$S_f=\left\{f_1,f_2,\dots,f_n\right\}$ and
	$S_l=\left\{l_1,l_2,\dots,l_n\right\}$ are solutions to $\mathcal{P}$.

	Suppose $S_f$ is not a solution. Then there must exist a constraint $C_{ij}$
	such that $\left(f_{i},f_{j}\right)\notin C_{ij}$. since $f_i \in
	D_i^\prime$, there exists a smallest available support say $v_j\in
	D_j^\prime$ such that $(f_i,v_j)\in C_{ij}$. Similarly, since $f_j \in
	D_j^\prime$, there exists a smallest available support say $v_i\in
	D_i^\prime$ such that $(v_i,f_j)\in C_{ij}$. From the definition of
	$f_i,f_j$, it follows that since $(f_i,f_j)\notin C_{ij}$, hence $f_i<v_i$
	and $f_j<v_j$. Since $C_{ij}$ is min-closed,
	$(\min\{f_i,v_i\},\min\{f_j,v_j\})\in C_{ij}$, i.e., $(f_i,f_j)\in C_{ij}$,
	which is a contradiction.

	Therefore, $S_f$ is a solution to $\mathcal{P}$.	
	Similarly, we can show that $S_l$ is also a solution to $\mathcal{P}$.
\end{proof}

\subsection{ACiDS Algorithm \label{sec:ACiDS}}

\subsubsection{Algorithm Overview}

We now present the algorithm ACiDS (pseudocode in Algorithm~\ref{alg:ACiDS} with
the \emph{ArcCons()} and \emph{LocalArcCons()} routines in
Algorithm~\ref{alg:ACiDS-arc}) that computes arc consistency for a DSCSP.  This
algorithm is a specialization of the generic AC5 algorithm \cite{GenericAC5}. At
a high level, the algorithm first detects the values that are without any
support: $\{v_j\in D_j|\nexists v_i \in D_i:(v_j,v_i)\in C_{ji}\}$. Then for all
$k$ such that $C_{kj}\in \mathcal{C}$, the tuple $(k,j,v_j)$ is added into a
queue $Q$ which essentially holds values due for constraint propagation while
pruning.  After dequeuing the tuple $(k,j,v_j)$ at a later stage, it checks
whether $v_j$ was the last available support for any $v_k \in D_k$, and puts all
such values into $Q$. The process terminates when $Q$ becomes empty.

\subsubsection{The MIN Data Structure}

For efficiently computing the set of values in $D_i$ for which $v_j$ is the only
support, we introduce a new data structure $MIN$. We define $MIN(v_j,i)$ as the
set of values in $D_i$ for which $v_j$ is the smallest available support in
$D_j$ over the constraint $C_{ij}$. Initially, we set
$MIN(v_j,i)=\{v_i|v_j=min(C_{ij},v_i)\}$. If $\nexists v_i \in D_i$ such that
$v_j=min(C_{ij},v_i)$, then $MIN(v_j,i)=\phi$. In Figure~\ref{fig:staircase}(a),
$MIN(b_1,i)=[a_1,a_2]$, $MIN(b_3,i)=[a_3,a_4]$, $MIN(b_5,i)=[a_5]$, and
$MIN(b_2,i)=MIN(b_4,i)=\phi$.  Due to the DS property of $C_{ji}$, the set
$MIN(v_j,i)$ can be represented as an interval. To illustrate this, suppose
$MIN(v_j,i)=\{v_1,v_3,v_6\}$ where $v_i \in D_i$ for $1\le i \le 6$. Since
$C_{ij}$ is a DS constraint and $min(v_1)=min(v_3)=min(v_6)$, it implies that
$v_2,v_4$ and $v_5$ does not have any support in $D_j$ over the constraint
$C_{ij}$. Hence, we can also denote $MIN$ in an interval form as
$MIN(v_j,i)=[v_1,v_6]$ which essentially marks the boundaries of $MIN(v_j,i)$.

We observe that since $C_{ij}$ is a DS constraint, for any $v_j,v_j^\prime \in
D_j$ such that $v_j<v_j^\prime$, if $MIN(v_j,i)=[a,b]$ and
$MIN(v_j^\prime,i)=[a^\prime,b^\prime]$, then $b<a^\prime$.  The algorithm
exploits this disjoint interval property of $MIN$ sets. During the course of the
algorithm, the $MIN$ sets and the $pred,succ$ values get updated, while $min$
and $max$ values remain unchanged.

\begin{algorithm}[!t]
\caption{ACiDS \label{alg:ACiDS}} 

\begin{algorithmic}[1]
{ \footnotesize \Procedure{main}{$\mathcal{P=\left(X,D,C\right)}$}

\State \emph{Initialize}$\left(\right)$

\ForAll{$C_{ij}\in\mathcal{C}$}

\State \emph{ArcCons}$\left(i,j,\triangle\right)$

\State \emph{Enqueue}$\left(i,\triangle,Q\right)$

\State \emph{Remove}$\left(\triangle,D_{i}\right)$

\EndFor 

\While{$Q\neq\phi$ \label{while:start}}

\State \emph{Dequeue}$\left(i,j,v_{j},Q\right)$

\State \emph{LocalArcCons}$\left(i,j,v_{j},\triangle\right)$

\State \emph{Enqueue}$\left(i,\triangle,Q\right)$

\label{while:end}
\State \emph{Remove}$\left(\triangle,D_{i}\right)$

\EndWhile

\EndProcedure

}
\end{algorithmic}

\begin{algorithmic}[1]
{ \footnotesize 
\Procedure{Initialize}{}

\State $Q\gets \left\{\right\}$
\State $\forall v_{j}\in D_{j},\ \forall C_{ij}\in\mathcal{C},MIN\left(v_{j},i\right)$$\gets \phi$

\ForAll{$C_{ij}\in\mathcal{C}$}

\State$scan_{i},scan_{j}$ point to the first value of $D_{i},D_{j}$ respectively

\While{$scan_{i}\neq\phi\wedge scan_{j}\neq\phi$}

\While{$min(scan_{i})>scan_{j}$}

\State $scan_{j}\gets min(scan_{i})$

\EndWhile

\State $MIN\left(scan_{j},i\right)$$\gets$ $MIN\left(scan_{j},i\right)\cup\left\{ scan_{i}\right\} $ 

\State $scan_{i} \gets succ\left(scan_i\right)$

\EndWhile

\EndFor

\EndProcedure
}
\end{algorithmic}

\begin{algorithmic}[1]
{
\footnotesize
\Procedure{Enqueue}{$i,\triangle, Q$}
\ForAll {$v_{i}\in\triangle$}
\State $Q \gets Q\cup\left\{ \left(j,i,v_{i}\right)|\forall\left(j,i\right)\in arc\left(\mathcal{P}\right)\right\} $
\EndFor
\EndProcedure
}
\end{algorithmic}

\begin{algorithmic}[1]
{
\footnotesize
\Procedure{Dequeue}{$i,j,v_j,Q$}

\If {$Q \neq \phi$} 

\State 
$Q \gets Q\setminus\left(i,j,v_{j}\right)$
\EndIf
\EndProcedure
}
\end{algorithmic}

\begin{algorithmic}[1]
{
\footnotesize
\Procedure{Remove}{$v,D_i$}
\ForAll {$C_{ij}\in \mathcal{C}$}

\If {$v \in D_i(C_{ij})$}

\State $pred\left(succ\left(v,D_i(C_{ij}) \right),D_i(C_{ij})\right)$

\State $\qquad$ $\gets$$pred\left(v,D_i(C_{ij})\right)$

\State $succ\left(pred\left(v,D_i(C_{ij})\right),D_i(C_{ij})\right)$

\State $\qquad$ $\gets$$succ\left(v,D_i(C_{ij})\right)$

\State$D_{i}(C_{ij})$$\gets$$D_{i}(C_{ij})\setminus\left\{ v\right\} $
\EndIf
\EndFor
\State $D_i \gets D_i \setminus\left\{ v\right\}$

\EndProcedure
}
\end{algorithmic}
\end{algorithm}

\subsubsection{Re-visiting ArcCons() and LocalArcCons()}

We re-design the generic subroutines \emph{ArcCons()} and \emph{LocalArcCons()},
stated in AC5, specifically for DSCSP in Algorithm~\ref{alg:ACiDS-arc}.
 
If $v_i \notin D_i(C_{ij})$, i.e., $v_i$ has no support in $D_j$ over the
constraint $C_{ij}$, then we define $min\left(C_{ij},v_i\right) =
max\left(C_{ij},v_i\right) = \phi$. We use pointer $scan_i$ to scan the values
of $D_i$. When $scan_i$ points to $v_i$, we say $scan_i=v_i$. Further,
$scan_i=\phi$ if and only if $scan_i$ reaches the end of domain $D_i$.
Similarly, we assume that $succ(v_i)=\phi$ if and only if $v_i$ is the largest
value in $D_i$.

The algorithm begins with initialization of data structures: queue $Q$ to an
empty queue and $MIN$ as per the definition stated above. Following this,
\emph{ArcCons}$(i,j,\triangle)$ is called for each constraint $C_{ij}\in
\mathcal{C}$.  $ArcCons(i,j,\triangle)$ identifies those values $v_i \in D_i$
such that $v_i \notin D_i(C_{ij})$ and adds them to the set $\triangle$.
Further, $Enqueue(i,\triangle,Q)$ enqueues tuples of the form $(j,i,v_i)$ into
$Q$ for all $arc(j,i) \in arc(\mathcal{P})$ and all $v_i \in \triangle$.  Then,
the call to $Remove(\triangle,D_i)$ removes all $v_i \in \triangle$ from $D_i$
and $D_i(C_{ij})$ for each constraint $C_{ij}$. In addition, it resets the
$pred$ and $succ$ values for the preceding and succeeding values of $v_i$ in
$D_i$.
 
The algorithm continues until $Q$ becomes empty. Within the while loop (lines
\ref{while:start}-\ref{while:end}), a tuple $(i,j,v_j)$ is dequeued from $Q$
using $Dequeue(i,j,v_j,Q)$ and then the subroutine
$LocalArcCons(i,j,v_j,\triangle)$ is invoked. 

The heart of the ACiDS algorithm lies in the subroutine
\emph{LocalArcCons}$(i,j,v_j,\triangle)$. Firstly, it determines the set
$\triangle=\{v_i \in MIN(v_j,i)|succ(v_j)=\phi \vee max(v_i)<succ(v_j)\}$ as the
set of values $v_i \in D_i$ for which $v_j \in D_j$ was the only available
support. Since the constraint is a DS, these values of $\triangle$ can be
identified by sequentially accessing the values in the $MIN$ set using the
$succ$ values, without visiting any $v\notin \triangle$.  Hence, this step
requires at most $O(\triangle)$ operations.  Secondly, if
$v_j^\prime=succ(v_j)\neq \phi$, it adds the set
$\left\{MIN(v_j,i)-\triangle\right\}$ to the set $MIN(v_j^\prime,i)$. This is to
ensure that those values  $v_i\in D_i$ for which $v_j$ was the smallest
available support in $D_j$ but not the only remaining support, will henceforth
have $v_j^\prime$ as the smallest available support. Since $MIN$ sets are stored
as disjoint intervals, this step can be completed in $O(1)$ time.  It is
interesting and important to note here that if $v_j$ is not the smallest support
for any $v_i$, i.e., $MIN(v_j,i)=\phi$, then no data structure needs to be
updated and $LocalArcCons()$ returns immediately. 

\begin{algorithm}[!t]
\caption{ArcCons and LocalArcCons \label{alg:ACiDS-arc}} 

\begin{algorithmic}[1]
{
\footnotesize 
\Procedure{ArcCons}{$i,j,\triangle$}
\State $\triangle \gets \left\{ v_{i}\in D_i|\nexists v_j \in D_j:(v_i,v_j) \in C_{ij}\right\}$
\EndProcedure

}
\end{algorithmic}

\begin{algorithmic}[1]
{
\footnotesize

\Procedure{LocalArcCons}{$i,j,v_j,\triangle$}

\If{$MIN\left(v_{j},i\right)=\phi$}

\State \textbf{return}

\EndIf

\State$\triangle$$\gets \phi$

\State $MIN\left(v_{j},i\right) = \left[v_{i_{1}},v_{i_{2}}\right]$ for some $v_{i_{1}},v_{i_{2}}\in D_{i}$ \label{l1}

\State$v_{j}^{\prime} \gets succ(v_j)$

\State $scan_{i} \gets v_{i_{1}}$

\While{$\left(v_{j}^{\prime}=\phi \vee  max\left(C_{ij},scan_{i}\right) < v_{j}^{\prime} \right)\wedge \left(scan_{i}\leq v_{i_{2}}\right)$}

\State$\triangle$$\gets$$\triangle\cup\left\{ scan_{i}\right\} $

\State$scan_{i} \gets succ(scan_{i})$

\EndWhile

\If{$scan_i \neq \phi \wedge scan_{i}\leq v_{i_{2}}$}

\State $MIN\left(v_{j}^{\prime},i\right)$$\gets$$MIN\left(v_{j}^{\prime},i\right)\cup\left[scan_{i},v_{i_{2}}\right] $ \label{x}

\EndIf \label{l2}

\EndProcedure
}
\end{algorithmic}

\end{algorithm}

\subsubsection{Correctness of ACiDS Algorithm}
 
In order to establish the correctness of the ACiDS algorithm which is a
specialization of AC5, we essentially need to show the correctness of
$LocalArcCons()$. In order to do so, we first show that $MIN(v_j,i)$ always
retains the set of values in $D_i$ for which $v_j$ is the smallest available
support in $D_j$ over the constraint $C_{ij}$.

\begin{lem}
	\label{lem:CorrectnessLocalArcCons}  
	At the end of any iteration of LocalArcCons$(i,j,v_j,\triangle)$ over a
	given ordered pair $(i,j)$, for any $v_{i}\in D_{i}\setminus \triangle$ and
	$v_{j}\in D_{j}$, $v_{i}\in MIN(v_{j},i)$ if and only if $v_{j}$ is the
	smallest available support for $v_{i}$ over the constraint $C_{ij}$.
\end{lem}

\begin{proof}
	We prove this lemma by applying induction on iteration $r$.
	
	Consider any $v_{i}\in D_{i}$. Suppose $min(C_{ij},v_{i})=v_{j}$.  Note that
	during initialization, $v_{i}\in MIN(v_{j},i)$.  Now, consider the call to
	\emph{LocalArcCons}$(i,j,v_{j_{r}},\triangle)$ for the first time ($r=1$)
	over the pair $(i,j)$. If $v_{j_{1}}\neq v_{j}$, $v_{i}$ continues to be in
	$MIN\left(v_{j},i\right)$ and $v_{j}$ continues to be the smallest available
	support for $v_{i}$. If, however, $v_{j_{1}}=v_{j}$, then the following two
	cases arise: \\
	(1) If $succ(v_{j})=\phi$ or $max(C_{ij},v_{i}) < succ(v_{j})$ such that
	$succ(v_{j}) \neq \phi$, then using the DS property, it implies that $v_{i}$
	lost its only support $v_{j}$ and is, thus, inserted to $\triangle$.  In
	this case, the hypothesis is no longer applicable as $v_i \notin D_i
	\setminus \triangle$. \\
	(2) If $max(C_{ij},v_{i}) \ge succ(v_{j})$ such that $succ(v_{j}) \neq
	\phi$, then $v_{i}$ is added to $MIN(v_{j^{\prime}},i)$ (line \ref{x}) where
	$v_{j^\prime} =$ $succ(v_{j})$.  By definition of $succ()$, $v_{j^{\prime}}$
	is the next greater value after $v_{j}$ in $D_{j}(C_{ji})$.  Since
	$v_{j}=min(C_{ij},v_{i})$ is already pruned from $D_{j}$, $v_{j}^{\prime}$
	is the smallest available support for $v_{i}$ over the constraint $C_{ij}$. 

	Hence, the hypothesis holds for $r=1$.

	Suppose the hypothesis holds for $r=p-1$. Now consider the call to
	\emph{LocalArcCons}$(i,j,v_{j_{r}},\triangle)$ for the $r=p^{th}$ iteration
	over the pair $(i,j)$. If $v_{j_{p}}\neq v_{j}$, then following the
	induction hypothesis, $v_{i}$ continues to be in $MIN\left(v_{j},i\right)$
	and $v_{j}$ continues to be the smallest available support for $v_{i}$. If,
	however, $v_{j_{p}} = v_{j}$, then the following two cases arise: \\
	(1) If $succ(v_{j})=\phi$ or $max(C_{ij},v_{i}) < succ(v_{j})$ such that
	$succ(v_{j}) \neq \phi$, then using the DS property, it implies that $v_{i}$
	lost its only support $v_{j}$ and is, thus, inserted to $\triangle$.  In
	this case, the hypothesis is no longer applicable as $v_i \notin D_i
	\setminus \triangle$. \\
	(2) If $max(C_{ij},v_{i}) \ge succ(v_{j})$ such that $succ(v_{j}) \neq
	\phi$, then $v_{i}$ is added to $MIN(v_{j^{\prime}},i)$ (line \ref{x}) where
	$v_{j^\prime} = succ(v_{j})$.  By definition of $succ()$, $v_{j^{\prime}}$
	is the next greater value after $v_{j}$ in $D_{j}(C_{ji})$.  Following the
	induction hypothesis, for $r=p-1$, since $v_{j}$ was the smallest available
	support for $v_{i}$ over the constraint $C_{ij}$ and now $v_{j}$ is already
	pruned from $D_{j}$, $v_{j^{\prime}}$ is the smallest available support for
	$v_{i}$ at the end of the $r=p^{th}$ induction step.

	Thus, the induction hypothesis follows.
\end{proof} 

The next lemma proves the correctness of the subroutine \emph{LocalArcCons()}.

\begin{lem}
	\label{lem:CorrectnessACiDS}
	LocalArcCons$(i,j,v_j,\triangle)$ correctly computes the set: $\triangle =
	\{ v_{i}\in D_{i}|\ (v_{i},v_{j}) \in C_{ij},\forall v_{j}^{\prime}\in
	D_{j},\left(v_{i},v_{j}^{\prime}\right) \notin C_{ij}\}$.
\end{lem}

\begin{proof}
	From Lemma \ref{lem:CorrectnessLocalArcCons}, $MIN(v_j,i)=\phi$, if and only
	if $v_j$ is not the smallest available support for any $v_i \in D_i$. Hence,
	pruning $v_j$ would not make any $v_i$ inconsistent. Therefore, the algorithm
	returns immediately.
	
	Now consider the case $MIN(v_j,i) \neq \phi$. Note that any $v_i$ that is
	supported by $v_j$ and has $v_j$ as the only available support in $D_j$ must
	lie in $MIN(v_j,i)$. From the pseudocode (lines \ref{l1}-\ref{l2}), we note
	the following two cases:\\ (1) If $succ(v_j)=\phi$, then using Lemma
	\ref{lem:CorrectnessLocalArcCons}, there exists no $v_j^\prime$ larger than
	$v_j$ in $D_j$ that can potentially support any $v_i$. Hence each $v_i \in
	MIN(v_j,i)$ is added to $\triangle$.  \\
	(2) If $succ(v_j)\neq \phi$, then suppose $MIN(v_j,i) = [ v_{i_1},\dots,$
	$v_{i_2},v_{i_3},\dots,v_{i_4}]$, such that $max(v_{i_2})<succ(v_j)$ and
	$max(v_{i_3}) \ge succ(v_j)$.  Consider the set of values $v_i \in
	[v_{i_1},v_{i_2}]$. From Lemma \ref{lem:CorrectnessLocalArcCons}, $v_j$ is
	the smallest available support for each of these $v_i$.  Further, since
	$C_{ij}$ is a DS, $max(v_i)<succ(v_j)$. Hence  after pruning $v_j$, no $v_i$
	has any support in $D_j$ and is thus added to $\triangle$.   Also observe
	that since $max(v_{i_3}) \ge succ(v_j)$, any $v_i^\prime \in
	[v_{i_3},v_{i_4}]$ has $succ(v_j)$ as a valid support after pruning $v_j$.
	Hence, no $v_i^\prime$ is added to $\triangle$. 
\end{proof}
 
\subsubsection{Complexity of ACiDS Algorithm}

We first analyze the time complexity of ACiDS. The subroutine
\emph{Initialize()} runs in $O(cd)$ time. For each constraint $C_{ij} \in
\mathcal{C}$, this procedure simultaneously scans the domains $D_i$ and $D_j$
without revisiting any value.  Next, we note that $ArcCons(i,j,\triangle)$ runs
in $O(d)$ time. From the description of \emph{LocalArcCons}$()$ in Section
\ref{sec:ACiDS}, we argue that it requires $O(\triangle)$ time. 

It was stated in \cite{GenericAC5} that if the time complexity of
\emph{ArcCons()} is $O(d)$ and that of \emph{LocalArcCons()} is $O(\triangle)$,
then the algorithm AC5 runs in $O(cd)$ time.  Since the $ArcCons()$ and
$LocalArcCons()$ in ACiDS take $O(d)$ and $O(\triangle)$ time, we conclude that
ACiDS also runs in $O(cd)$ time.

Analyzing the space complexity, we note that since each DS constraint can be
stored as a list of intervals, the total input space requirement for a DSCSP is
$O(cd)$. During the execution, since any $v_i \in D_i$ is enqueued at most once
for each $arc(j,i) \in arc \left(\mathcal{P}\right)$, the queue length of $Q$ is
at most $O(cd)$. By a similar argument, the $MIN$ data structure takes at most
$O(cd)$ space.  Since at any point, $\triangle$ holds values of any one domain,
$\triangle$ takes no more than $O(d)$ space. Hence, the ACiDS algorithm uses at
most $O(cd)$ space.

The $O(cd)$ space and time complexity is \emph{optimal} for any CSP with $c$
constraints, since each constraint needs $\Omega(d)$ space, and at least one
scan of any domain requires $\Omega(d)$ time.  The ACiDS algorithm improves the
complexity bound of $O(cd^2)$ in \cite{AAAI:2007RowConvex} for the subclass DS
of CRC constraints. 

\subsection{Arc Consistency for USCSP}

\begin{algorithm}[t]
\caption{Subroutines for Arc Consistency for USCSP\label{alg:ACUS}}

\begin{algorithmic}[1]
{ \footnotesize 
\Procedure{Initialize}{}

\State $Q\gets \left\{\right\}$
\State $\forall v_{j}\in D_{j},\ \forall C_{ij}\in\mathcal{C},MIN\left(v_{j},i\right)$$\gets \phi$

\ForAll{$C_{ij}\in\mathcal{C}$}

\State$scan_{i},scan_{j}$ point to last and first values of $D_{i},D_{j}$
respectively

\While{$scan_{i}\neq\phi\wedge scan_{j}\neq\phi$}

\While{$min(scan_{i})>scan_{j}$}

\State $scan_{j}\gets min(scan_{i})$

\EndWhile

\State $MIN\left(scan_{j},i\right)$$\gets$ $MIN\left(scan_{j},i\right)\cup\left\{ scan_{i}\right\} $ 

\State $scan_{i} \gets pred\left(scan_i\right)$

\EndWhile

\EndFor

\EndProcedure
}
\end{algorithmic}

\begin{algorithmic}[1]
{
\footnotesize

\Procedure{LocalArcCons}{$i,j,v_j,\triangle$}

\If{$MIN\left(v_{j},i\right)=\phi$}

\State \textbf{return}

\EndIf

\State$\triangle$$\gets \phi$

\State $MIN\left(v_{j},i\right) = \left[v_{i_{1}},v_{i_{2}}\right]$ for some $v_{i_{1}},v_{i_{2}}\in D_{i}$ \label{l1}

\State$v_{j}^{\prime} \gets succ(v_j)$

\State $scan_{i} \gets v_{i_{2}}$

\While{$\left(v_{j}^{\prime}=\phi \vee  max\left(C_{ij},scan_{i}\right) < v_{j}^{\prime} \right)\wedge \left(scan_{i}\geq v_{i_{1}}\right)$}

\State$\triangle$$\gets$$\triangle\cup\left\{ scan_{i}\right\} $

\State$scan_{i} \gets pred(scan_{i})$

\EndWhile

\If{$scan_i \neq \phi \wedge scan_{i}\geq v_{i_{1}}$}

\State $MIN\left(v_{j}^{\prime},i\right)$$\gets$$MIN\left(v_{j}^{\prime},i\right)\cup\left[v_{i_{1}},scan_{i}\right] $ \label{x}

\EndIf \label{l2}

\EndProcedure
}
\end{algorithmic}
\end{algorithm}

In this section, we show that with minor modifications in the way we scan the
domains in the subroutines, \emph{Initialize()} and \emph{LocalArcCons()}, the
ACiDS algorithm can compute arc consistency for a USCSP.  The modified
subroutines are presented in Algorithm~\ref{alg:ACUS}. In addition to the
assumptions stated in the ACiDS algorithm, we assume $scan_i=\phi$ if and only
if $scan_i$ reaches the end of the domain on either side, and $pred(v_i)=\phi$
if and only if $v_i$ is the smallest value of the domain.  It is easy to follow
that arc consistency for a USCSP can also be achieved with the same time and
space complexity of $O(cd)$ as in the case of DSCSP.  However, since arc
consistency is not known to be sufficient to determine a solution to a USCSP,
techniques such as path consistency or variable elimination are required. Since
arc consistency is a basic subroutine in these techniques, replacing that by our
modified ACiDS algorithm produces a faster solution. It improves the running
time of $O\left(\sigma^{2}nd+cd^{2}\right)$ (where $\sigma \le n$ is a problem
specific parameter) of the best known algorithm \cite{AAAI:2007RowConvex} to
$O\left(\sigma^{2}nd+cd\right)$, which is linear in $d$.

\section{A Faster Solver for DSCSP}
\label{sec:A faster algorithm}

\begin{algorithm}[t]
\caption{ DSCSP Solver\label{alg:DSCSP}}
\begin{algorithmic}[1]
{
\footnotesize
\Procedure{main}{$\mathcal{P=\left(X,D,C\right)}$}
\State \emph{Initialize()}

\While{$\forall i \ scan_i\neq\phi$ }

\State $j \gets$\emph{Dequeue} $\left(Q\right)\label{dequeue}$
\Comment{$flag_j= 0$ now}
\ForAll{$\{k:C_{jk}\in\mathcal{C}\}\label{restart}$}

\While {$scan_j \neq \phi \wedge scan_j \notin D_j(C_{jk})$} \Comment{$flag_j=0$ now}
\State Increment $scan_j$ to the next value in $D_j$
\EndWhile

\While{$ scan_k \neq \phi \wedge scan_k \notin D_k(C_{kj})$} \Comment{$flag_k=0$ now}
\State Increment $scan_k$ to the next value in $D_k$

\EndWhile

\If {$scan_j < min\left(C_{kj},scan_k\right)$}

\State $scan_j\gets min\left(C_{kj},scan_k\right)$

\State \textbf{continue} \Comment{Restart for loop}

\ElsIf {$scan_k < min\left(C_{jk},scan_j\right)$}

\State $scan_k \gets min\left(C_{jk},scan_j\right)$

\If {$flag_k=1$}

\State $flag_k \gets 0 \label{flag_decrement}$ 

\State $sum \gets sum -1$

\State $Q \gets Q \cup \left\{ k\right\} \label{enqueue}$

\EndIf

\EndIf

\EndFor \label{endfor}

\State $flag_j\gets 1 \label{flag_increment}$
\State $sum \gets sum +1$
\If {$sum=n$}
\State \textbf{print} Solution $=\{\forall i,\ scan_i\}$
\State \textbf{break} 
\EndIf

\EndWhile
\EndProcedure
}
\end{algorithmic} 

\begin{algorithmic}[1]
{
\footnotesize
\Procedure{Initialize}{}
\State $sum \gets 0$
\State $Q \gets \varnothing$
\ForAll {$X_i\in \mathcal{X}$}
\State $scan_i$ points to the first value of $D_i$
\State $flag_i \gets 0$
\State $Q \gets Q \cup \left\{ i\right\}$
\EndFor

\EndProcedure
}
\end{algorithmic}

\end{algorithm}

This section presents a more efficient algorithm \textsc{DSCSP Solver}
(Algorithm~\ref{alg:DSCSP}) to solve a DSCSP than ACiDS. We observe that a
solution to a DSCSP can be found even without achieving arc consistency.
Exploiting the DS property of the DS constraints and the total ordering of the
variable domains, the algorithm incrementally scans the variable domains in a
manner such that no value is revisited. At each point, we check whether the
values being scanned form a solution. Only when it is confirmed that the current
value cannot form a solution is the next value scanned.

The algorithm maintains a queue $Q$ of variable indices. A variable index $i \in
Q$ if and only if it is needed to check whether $(scan_i,scan_j)\in C_{ij}$ for
all the constraints $C_{ij}$, imposed on $X_i$.  The algorithm also uses
pointers $scan_i$ and boolean variables $flag_i$ for each variable $X_i \in
\mathcal{X}$. The variable $sum$ maintains the sum of all the $flag$ variables.
As stated earlier, we assume that for each pair of variables $(X_i,X_j)$, there
is at most one constraint $C_{ij}$, and we need not consider its transpose
$C_{ji}$ separately. (In the ACiDS algorithm, we considered both $C_{ij}, C_{ji}
\in \mathcal{C}$.)

DSCSP Solver first calls \emph{Initialize()} for initializing the $scan_i$
pointer to point to the first value of $D_{i}$ and $flag_i$ to $0$ for each
$X_{i}\in\mathcal{X}$. The queue $Q$ is initialized to contain all the variable
indices. 

Then, the algorithm proceeds in iterations.  In each iteration, we dequeue a
variable index $j$ and consider all the constraints $C_{jk}\in \mathcal{C}$. We
keep incrementing $scan_j$ (and $scan_k$) until it reaches a value which has a
support in $D_k$ ($D_j$ respectively). If $scan_j<min(scan_k)$, we set $scan_j =
min(scan_k)$ and reconsider all the constraints imposed on $X_j$.  Also, if
$scan_k<min(scan_j)$, we set $scan_k=min(scan_j)$.  Further, if $flag_k=1$, we
reset it to $0$ and add $k$ to $Q$. Since $sum$ maintains
$\sum_{i=1}^{n}flag_{i}$, $sum$ is decremented by $1$ as a consequence. At the
end of the for loop (lines \ref{restart}-\ref{endfor}), if $scan_j,scan_k \neq
\phi$, then for $scan_j$, every constraint $C_{jk}$ imposed on $X_j$ is
satisfied by $scan_k$.  As a result, $flag_j$ moves from $0$ to $1$ and,
therefore, $sum$ is incremented by $1$. Whenever $sum=n$, it implies $\forall j
\ flag_j=1$, which indicates that $(scan_j,scan_k)\in C_{jk}$ for all
constraints $C_{jk}\in \mathcal{C}$.  This in return signifies that at this
point, the assignment $X_i = scan_i$ is a solution to the DSCSP. 

\subsection{Correctness of DSCSP Solver}

To prove the correctness of the algorithm, we first show that the answer
reported by the algorithm is indeed a solution.  The following simple claims
lead to this fact.

\begin{lem}
	At the end of the for loop in lines \ref{restart}-\ref{endfor}, $flag_j=1
	\ \forall j$ if and only if $j\notin Q$.
\end{lem}

\begin{proof}
	Consider a variable index $j$ in $Q$. After $j$ is dequeued from $Q$, the
	for loop in lines \ref{restart}-\ref{endfor} is processed. At the end of the
	for loop, $flag_j$ is set to $1$. On the other hand, before any variable
	index $k$ is added to $Q$ in line \ref{enqueue}, $flag_k$ is decremented to
	$0$ in line \ref{flag_decrement}.
\end{proof}

\begin{lem}
	At any stage of the algorithm, $sum=\sum_{i=1}^{n}flag_{i}$.
\end{lem}

\begin{proof}
	From the previous lemma, it is clear that when any variable index $j$ is
	dequeued from $Q$ in line \ref{dequeue}, $flag_j=0$. Later, $flag_j$ is set
	to $1$ in line \ref{flag_increment}. And immediately, $sum$ is also
	incremented. On the other hand, when for any variable index $k$, whenever
	$flag_k$ is decremented from $1$ to $0$ in line \ref{flag_decrement}, $sum$
	is also decremented by $1$.
\end{proof}

 \begin{lem} \label{lem:DSCSPSolver1}
	When $flag_j$ moves from $0$ to $1$ (in line \ref{flag_increment}), then
	$(scan_j,scan_k)\in C_{jk}$ for each constraint $C_{jk}$  on $X_j$.
\end{lem}

\begin{proof}
	At a high level, the proof follows from the observation that before $flag_j$
	gets incremented in line \ref{flag_increment}, every constraint on $X_j$ is
	verified to be true for the current value of the $scan$ pointers.

	Consider the for loop in lines \ref{restart}-\ref{endfor}. Consider any
	constraint $C_{jk}\in \mathcal{C}$. Note that during the execution of the
	algorithm, any $scan$ pointer is never decremented. Following the \emph{down
	staircase} property of $C_{jk}$, it is clear that if $(scan_j,scan_k)\notin
	C_{jk}$ such that $scan_j\in D_j(C_{jk})$ and $scan_k \in D_k(C_{kj})$,
	either of the following two cases are possible but not both: (i)
	$scan_j<min\left( C_{kj},scan_k\right)$, or (ii) $scan_k<min\left(
	C_{jk},scan_j\right)$.
	
	From the algorithm, in case (i), $scan_j$ is incremented to $min\left(
	C_{kj},scan_k \right)$. This ensures that $(scan_j,scan_k)\in C_{jk}$.
	However, since $scan_j$ has been modified, we need to reconsider all the
	constraints $C_{jk}$ on $X_j$.  Arguing similarly for case (ii), it is clear
	that at the end of the for loop in lines \ref{restart}-\ref{endfor}, when
	$flag_j$ is set to $1$,  $(scan_j,scan_k)\in C_{jk}$ for each constraint
	$C{jk}$ on $X_j$.
\end{proof}

\begin{cor}
	At any stage during the execution of the algorithm, (i) for any constraint
	$C_{jk}$, if both $flag_j=1$ and $flag_k=1$, then $(scan_j,scan_k)\in
	C_{jk}$, and (ii) if $sum=n$, then there is a solution, namely, $\forall i \
	X_i=scan_i$.
\end{cor}

Now we show that the algorithm would never miss a solution, if it exists. The
key point to note here is that $scan_{i}$ is not incremented until it is
confirmed that the current value (being pointed by $scan_{i}$) cannot be a
member of any solution to the DSCSP. The following lemma formally proves this
claim. 
 
\begin{lem}
	If $v_i \in D_i$ is a member of a solution, then at every stage during the
	execution of the algorithm, $scan_i \leq v_i$.
\end{lem}

\begin{proof}
	Suppose $v_{i}$ is the first value which got crossed by $scan_i$ during
	the execution of the algorithm even when it is a member of a solution $S$.
	Suppose the solution is $S=\{\forall k, \ X_k=v_k\}$.
	
	Consider the time when $scan_{i}$ is about to be incremented from $v_{i}$.
	Observe that at this time, $scan_{i}=v_{i}$ and $scan_{j}\leq v_{j}$ for all
	$j\neq i$.  The value $v_{i}$ can get crossed if and only if for some constraint
	$C_{ij}$, either $v_i \notin D_i(C_{ij})$ or $v_i<min(scan_j)$. If
	$v_i \notin D_i(C_{ij})$, then $v_i$ cannot be a member of a solution, and
	hence, this is a contradiction. Considering the latter case, suppose
	$scan_j=v_j^\prime < v_j$. Since $(v_j,v_i)\in C_{ji}$, $min(v_j)
	\leq v_i$. However, $min(v_j^\prime)>v_i \geq min(v_j)$
	contradicts the DS property of $C_{ji}$.
\end{proof} 
	
Hence, we have shown that the solution reported by the DSCSP Solver is valid and
that it would never miss to detect a solution, if one exists. This concludes the
correctness analysis.
 	
\subsection{Complexity of DSCSP Solver}

From the above results, it follows that if $\forall i, v_i \in D_i$ is the
smallest value in $D_i$ which is a member of a solution, then the algorithm
would report the solution $\forall i, X_i=v_i$. The complexity analysis
leverages this fact.
		
\begin{thm}
	\label{thm:Complexity-of-DSCSPSolver}
	If $\delta_j$ is the degree of node $V_j$ (corresponding to variable $X_j$)
	in the constraint network, and $v_{s_j}$ is the  smallest value in $D_j$
	that is a member of a solution, then the algorithm DSCSP Solver takes time
	$O\left(\sum_{j=1}^{n}\delta_{j}s_{j}\right)=O\left(cs\right)$ where
	$s=max_{j}\left\{ s_{j}\right\}$.
\end{thm}
	
\begin{proof}
	For a given variable $X_j$ that is dequeued from $Q$, the maximum number of
	iterations of the for loop in lines \ref{restart}-\ref{endfor} is $s_j$ as
	the number of increments of $scan_j$ is at most $s_j-1$. Now consider any
	given iteration of this for loop for the variable $X_j$.  Suppose
	$scan_j=v_j$. Before $scan_j$ gets incremented during this iteration, $v_j$
	is compared with at most $scan_k$ for all $k$ such that $C_{jk} \in
	\mathcal{C}$. This implies that in this iteration, $v_j$ is compared with at
	most $\delta_j$ values. Hence, the result follows.
\end{proof}

In the absence of any solution, the algorithm takes $O(cd)$ time because at
least one of the $scan$ pointers must scan the entire domain. Analyzing the
space complexity, besides the input which takes $O(cd)$ space, this algorithm
requires only $O(n)$ additional space for storing the $scan$ pointers.

Comparing the DSCSP Solver with the ACiDS algorithm, the key advantages are the
following: (i) DSCSP Solver is much simpler to implement, (ii) it need not scan
the entire domain if there is a solution, and hence, offers an improved time
complexity of $O(\sum_{j=1}^{n}\delta_{j}s_{j})$ as opposed to $O(cd)$ time in
presence of a solution, and (iii) its additional space requirement is $O(n)$ as
opposed to $O(cd)$ in the ACiDS algorithm.  

\section{Conclusions \label{sec:conclusion}}

This paper investigated the class of \emph{generalized staircase (GS)}
constraints, which generalizes the class of staircase constraints, while being a
subclass of CRC constraints. We explored the properties of the two subclasses
of GS, namely, \emph{down staircase (DS)} and \emph{up staircase (US)}
constraints, besides studying their relationships with the existing constraint
classes.

We proposed two algorithms for solving DSCSP, namely, ACiDS, based on arc
consistency, and DSCSP Solver, based on incremental scan of domains. We showed
that DSCSP Solver is more efficient as compared to ACiDS. For solving USCSP, we
require either path consistency or variable elimination technique. Our optimal
ACiDS algorithm acting as a subroutine in these techniques, offer a faster
solution to USCSP.

In the future, more efficient arc consistency scheme for CRC constraints may be
investigated. In addition, interesting generalizations of GS constraints and
applications to real life problems will be explored.

{
\bibliographystyle{abbrv}
\balance
\bibliography{biblio1}
}

\end{document}